\documentclass[10pt]{article}
\usepackage{fullpage}

\usepackage[dvipsnames]{xcolor}
\usepackage{hyperref}
\hypersetup{
   colorlinks=true,
   linkcolor=magenta,     
   citecolor= ForestGreen
}

\usepackage[T1]{fontenc}
\usepackage[english]{babel}
\usepackage{times}
\usepackage{bbm,pifont,amsmath,amssymb,amsthm,amsbsy,amsfonts,mathtools,cite,cleveref,framed,algorithm,algorithmic,refcount,enumerate}
\usepackage{colortbl, booktabs}
\usepackage{float,graphicx,floatflt}
\graphicspath{{fig/}}

\newcommand{\tmop}[1]{\ensuremath{\operatorname{#1}}}

\def \lg       {\langle}
\def \rg       {\rangle}

\newtheorem{proposition}{Proposition}[section]
\newtheorem{lemma}{Lemma}[section]
\newtheorem{theorem}{Theorem}[section]

\newtheorem{definition}{Definition}[section]

\theoremstyle{remark}
\newtheorem{remark}{Remark}[section]

\newcommand{\vct}[1]{\mathbf{#1}}
\newcommand{\mtx}[1]{\mathbf{#1}}

\newcommand{\set}[1]{\mathcal{#1}}

\def \mDelta  {{\boldsymbol{\Delta}}}

\def \mSigma  {{\boldsymbol{\Sigma}}}

\def \q {\vct{q}}

\def \u {\vct{u}}
\def \v {\vct{v}}

\def \x {\vct{x}}

\def \z {\vct{z}}

\def \vu {\vct{u}}

\def \mA {\mtx{A}}
\def \mB {\mtx{B}}
\def \mC {\mtx{C}}
\def \mD {\mtx{D}}

\def \mI {\mtx{I}}

\def \mM {\mtx{M}}

\def \mP {\mtx{P}}

\def \mU {\mtx{U}}
\def \mV {\mtx{V}}
\def \mW {\mtx{W}}
\def \mX {\mtx{X}}
\def \mY {\mtx{Y}}
\def \mZ {\mtx{Z}}

\def \calC {\set{C}}

\def \calO {\set{O}}

\def \bbU {\mathbb{U}}
\def \bbV {\mathbb{V}}

\DeclareMathOperator*{\minimize}{\operatorname{minimize}}
\DeclareMathOperator*{\maximize}{\operatorname{maximize}}
\DeclareMathOperator*{\argmin}{\operatorname{arg~min}}
\DeclareMathOperator*{\argmax}{\operatorname{arg~max}}

\def \zero     {\mathbf{0}}

\def \lg       {\langle}
\def \rg       {\rangle}

\def \tr       {\tmop{tr}}

\def \dist     {\tmop{dist}}

\def \R {\mathbb{R}}

\begin{document}

\title{Spherical Principal Component Analysis}
\author{Kai Liu\footnote{Equal contribution.}~~\thanks{Department of Computer Science, Colorado School of Mines.} \\
\and   
Qiuwei Li\footnotemark[1]~~\thanks{Department of Electrical Engineering, Colorado School of Mines.}\\
\and
Hua Wang\footnotemark[2]\\
\and
Gongguo Tang\footnotemark[3]
}
\date{}

\maketitle

\begin{abstract} 
 Principal Component Analysis (PCA) is one of
	the most important methods to handle high dimensional
	data. However, most of the studies on PCA aim to minimize the loss after projection, which usually measure the Euclidean distance, though in some fields, angle distance is known to be more important and critical for analysis. In this paper, we propose a method by adding constraints on factors to unify the Euclidean distance and angle distance. However, due to the nonconvexity of the objective and constraints, the optimized solution is not easy to obtain. We propose an alternating linearized minimization method to solve it with provable convergence rate and guarantee. Experiments on synthetic data and real-world datasets have validated the effectiveness of our method and demonstrated its advantages over state-of-art clustering methods. 
\end{abstract}

\section{Introduction}
In many real-world applications such as text categorization and face recognition, the dimensions of data are usually
very high. Dealing with high-dimensional data
is computationally expensive while noise or outliers in the data can
increase dramatically as the dimension increases. Dimension
reduction  is one of
the most important and effective methods to handle high dimensional
data~\cite{roweis2000nonlinear,tenenbaum2000global,belkin2003laplacian}. Among the dimension reduction
methods, Principal Component Analysis (PCA) is one of the
most widely used methods due to its simplicity and effectiveness.

PCA is a statistical procedure that uses an orthogonal transformation to convert a set of correlated variables into a set of  linearly uncorrelated principal directions. Usually the number of principal directions is less than or equal to the number of original variables. This transformation is defined in such a way that the first principal direction has the largest possible variance (that is, accounts for as much of the variability in the data as possible), and each succeeding direction has the highest variance under the constraint that it is orthogonal to the preceding directions. The resulting vectors are an uncorrelated orthogonal basis set.

When data points lie in a low-dimensional manifold and the
manifold is linear or nearly-linear, the low-dimensional
structure of data can be effectively captured by a linear
subspace spanned by the principal PCA directions.

More specifically, let $\mX =
[\x_1~\cdots~ \x_n]\in\R^{m\times n}$ be $n$ data points in $m$-dimensional space while $\mU = [\u_1~ \cdots ~ \u_r]\in\R^{m\times r}$ contains the principal directions
and $\mV = [\v_1~ \dots ~ \v_n]\in\R^{r\times n}$ contains the principal components
(data projects along the principal directions).
Generally speaking, there can be two formulations for PCA:
\begin{itemize}
	\item Covariance-based approach, which computes the covariance
	matrix 
$\mC =\sum_{i=1}^n(\x_i  - \bar{\x})(\x_i - \bar{\x})^\top= \mX\mX^\top.$
	Here
	we assume the data are already centered, i.e., $\bar{\x} = 0$, and
	we drop the factor $\frac{1}{n-1}$ which does not affect $\mU$.
	The principal directions are obtained as:
	\begin{equation}
	\label{eq:nmf_obj_pca1}
	\maximize_{\mU^\top\mU=\mI} \tr(\mU^\top\mX\mX^\top\mU).
	\end{equation}
	\item Matrix low-rank approximation-based approach.
	Let $\mX \approx \mU\mV$, we solve:
	\begin{equation}
	\label{eq:nmf_obj_pca2}
	\minimize_{\mU^\top\mU=\mI}  \|\mX-\mU\mV\|^2_F = \sum_{i,j}[\mX_{ij}-(\mU\mV)_{ij}]^2.
	\end{equation}
\end{itemize}
Taking the derivative w.r.t. $\mV$ and setting it to zero, we have
$\mV = \mX^\top\mU$, and Eq.~(\ref{eq:nmf_obj_pca2}) reduces to Eq.~(\ref{eq:nmf_obj_pca1}). Therefore, the solutions to these two approaches
are identical. In our paper, we mainly focus on the second formulation.

\section{Motivation}
In Eq.~(\ref{eq:nmf_obj_pca2}), the objective function measures the gap between original data $\mX$ and approximation after projection $\mU\mV$, which is based on squared Euclidean distance measurements and treat each feature as equally important. However, in the real world, there are some given datasets which are preprocessed to be normalized and different features may have various significance. Thus distance-based measurement method may yield poor results. On the other side, similarity-based measurement methods such as angle distance have been proved to be more efficient in some applications, including information retrieval~\cite{singhal2001modern}, signal processing~\cite{hou1987fast}, metric learning~\cite{nguyen2010cosine}, etc.. Though one can calculate the similarity  after projection, still this appears to be more or less awkward and inefficient. Thus, deriving some methods which can directly measure angle distance from PCA is vitally important. However, to our best knowledge, it has not been studied yet.
\begin{figure}
	\centering
    \begin{tabular}{cc}
	\includegraphics[width=0.35\linewidth]{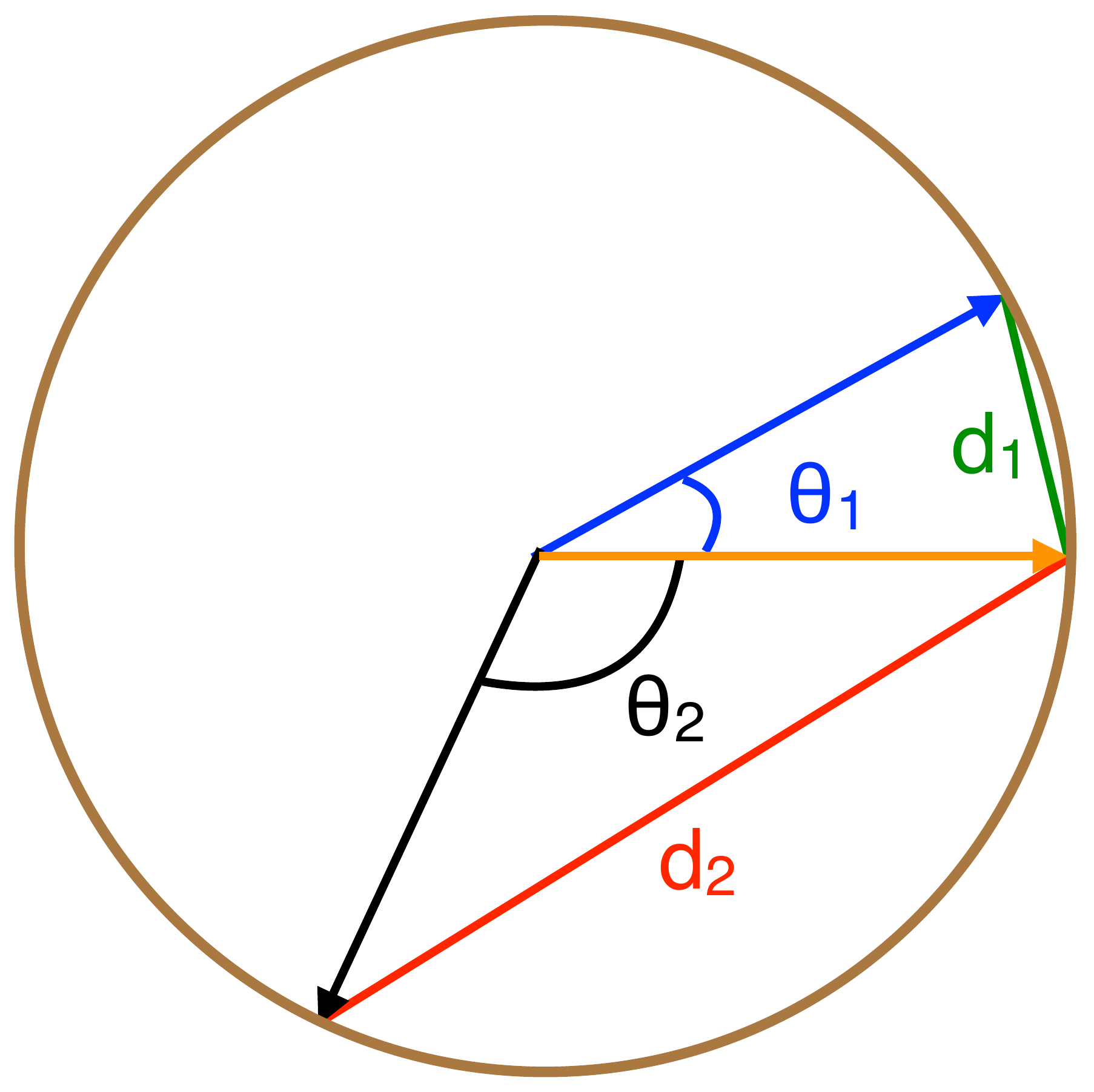} &
	\includegraphics[width=0.38\linewidth]{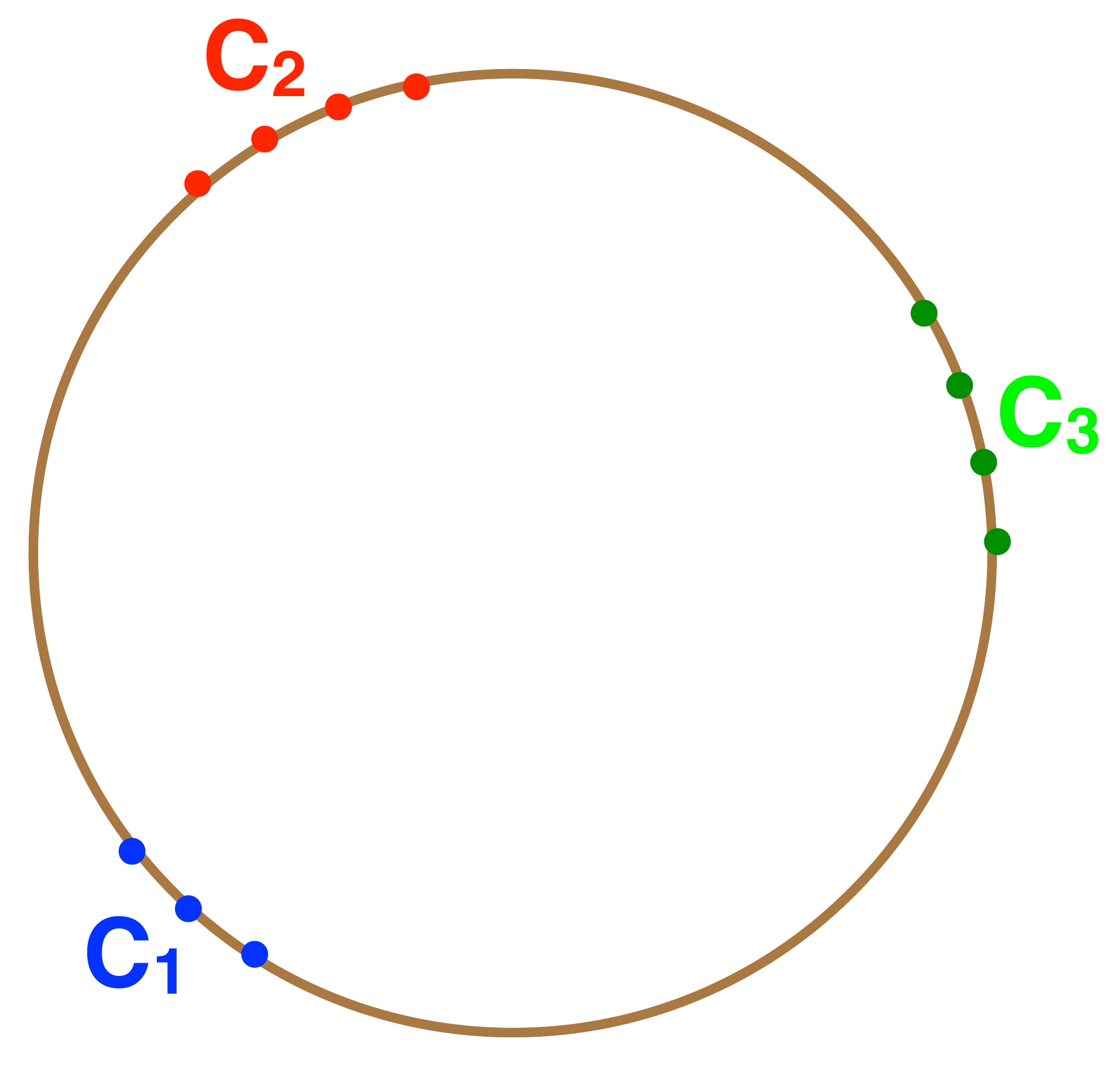}
    \end{tabular}
	\caption{Larger angles ($\theta_2>\theta_1$) in the sphere will have larger Euclidean distance, and vice versa, which unifies the cosine similarity and Euclidean distance simultaneously.}
	\label{fig:theta}
\end{figure}

Motivated by the above observations and a previous work \cite{sun2011angular}, in this paper we propose a spherical-PCA model which can unify the Euclidean distance and angle distance. By noticing that larger angle in the sphere in Fig.~\ref{fig:theta} also has larger Euclidean distance, we can add the normalization constraint to the component matrix, where the norm of each column in $\mV$ is 1 to guarantee the spherical distribution of components:
\begin{equation}
\label{eq:nmf_obj_pca3}
\begin{aligned}
\minimize_{\mU\in\R^{m\times r},\mV\in\R^{r\times n}} & \|\mX-\mU\mV\|^2_F = \sum_{i,j}[\mX_{ij}-(\mU\mV)_{ij}]^2 
\quad s.t. \quad \mU \in \bbU,\ \mV\in \bbV
\end{aligned}
\end{equation}
where we define:
\begin{equation}
\begin{aligned}
	\bbU&:=\{\mU: \mU^\top\mU = \mI\}, 
\\
	\bbV&:=\{\mV: \|\v_j\|=1 \ \forall j\},
	\label{eq:setFG 2}
\end{aligned}
\end{equation}
where $\|\cdot\|$ denotes $\ell_2$ norm for vectors  and denotes the spectral norm for matrices. Suppose the component is spherically distributed, then the Euclidean distance between $\v_i$ and $\v_j$ is:
\begin{equation}
\label{eq:nmf_obj_pca4}
\begin{aligned}
\|\v_i-\v_j\|^2 &= \|\v_i\|^2+\|\v_j\|^2-2\langle \v_i,\v_j\rangle\\
& = \|\v_i\|^2+\|\v_j\|^2-2\frac{\langle \v_i,\v_j\rangle}{\|\v_i\|\|\v_j\|} \\
&= 2-2 \cos(\theta),\ \theta \in [0,\pi]
\end{aligned}
\end{equation}
which is equivalent to angle distance that bigger angle $\theta$ will result in larger Euclidean distance, and vice versa.

\begin{remark}
	In traditional PCA, without the normalization constraint on each column of $\v$, the optimized solution to Eq.~(\ref{eq:nmf_obj_pca2}) can barely satisfy the spherical distribution. Since $r$ is usually less than $m$, PCA will lose some component more or less, thus $\x_i \neq \mU\v_i$ and usually $\|\x_i\| \neq \|\mU\v_i\|$ (they may be equal, but it barely happens) . We have $\|\x_i\|^2=1$ for normalized data and if $\|\v_i\|^2=1$ then $\|\mU\v_i\|^2=\tr (\v_i^\top\mU^\top\mU\v_i)=\tr (\v_i^\top\v_i)=\|\v_i\|^2=1$, which leads a contradiction, thus the constraint on $\mV$ is necessary to guarantee our motivation.
\end{remark}

\section{Formulation And Algorithm}
\subsection{Objective Function with Proximal Term}
We first denote:
\begin{equation}
\label{eq:h}
h(\mU,\mV) = \|\mX-\mU\mV\|_F^2=\sum_{j=1}^n \|\x_j-\mU\v_j\|^2 \quad s.t. \quad \mU \in \bbU, \mV\in \bbV
\end{equation}
By noting the nonconvexity of Eq.~(\ref{eq:nmf_obj_pca3}), where no closed solution exists, we propose an alternating minimization method to get the optimized solution as: given $k$th iterate of $\mV$ varaible $\mV(k)=[\v_1(k)~\cdots~\v_n(k)]$,
\begin{equation}
\begin{aligned}
	\label{eq:alm}
	\mU({k+1}) &= \argmin_{\mU\in\bbU} \|\mX-\mU\mV(k)\|_F^2;\\
	\v_j(k+1) &= \argmin_{\|\v\|=1}\|\x_j-\mU({k+1})\v\|_2^2, \ \forall j
\end{aligned}
\end{equation}

Note that when the constraints  $\mU \in \bbU, \mV\in \bbV
$, the problem \eqref{eq:h} is known as the nonconvex matrix factorization problems, which have been well-studied\cite{
nonsym:li2016,distr:zhu2018global}. This work focus on develop efficient and provable algorithm to deal with \eqref{eq:h} with  the constraints $\mU \in \bbU, \mV\in \bbV$. 
Note that the proximal algorithm recently
has been successfully applied to a wide variety of situations: convex optimization, nonmonotone operators ~\cite{combettes2004proximal,kaplan1998proximal} with various applications to nonconvex programming.
It was first introduced by Rockafellar~\cite{rockafellar1976augmented} as an approximation regularization method in convex optimization and in the study of variational inequalities associated to maximal monotone operators.

Considering the fact that the objective function in Eq.~(\ref{eq:nmf_obj_pca3}) is nonconvex w.r.t. $\mU$ and $\mV$, and the constraint on $\mU$ and $\mV$ are also nonconvex, we consider adding proximal term and optimize the solution as:
with the alternating linearized minimization solutions becomes: 
\begin{equation}
	\label{eq:palm}
\begin{aligned}
	\mU({k+1}) &= \argmin_{\mU\in\bbU} \langle \mU-\mU(k), \nabla_\mU h(\mU(k),\mV(k)) \rangle +\frac{\mu}{2}\|\mU-\mU(k)\|^2_F;\\
	\v_j(k+1) &= \argmin_{\|\v\|=1} \langle \v-\v_j(k), \nabla_{\v_j} h(\mU(k+1),\mV(k)) \rangle+\frac{\lambda}{2}\|\v-\v_j(k)\|^2, \ \forall j
\end{aligned}
\end{equation}
\begin{remark}
	We add the proximal term to make the new updating solution will not  be too far from the previous step to avoid drastic changes. One can see that when the proximal term regularization parameters $\mu,\lambda$ are sufficiently large, they will dominate the objective function. Moreover, we can take the linearized minimization as to minimize the objective with Taylor expansion by making use of first order (linear) information. 
\end{remark}

\subsection{Proposed Algorithm}
Given the alternating minimization objective in Eq.~(\ref{eq:palm}), now we turn to provide detailed (closed) updating algorithm.

 We first derive the solution for $\mU$ and before that we give a useful lemma that is similar to \cite[Theorem 1]{wang2013multi} and \cite[Theorem 1]{wang2011fast}:
 \begin{lemma}
 	$\maximize_{\mU^\top\mU=\mI} \tr (\mU^\top\mM)$ is given by $\mU=\mA\mB^\top$, where $[\mA,\mSigma,\mB] =\tmop{svd}(\mM)$.
 \end{lemma}
 \begin{proof}
 	On one hand, we have:
 	\begin{equation}
 	\label{eq:maxF}
 	\tr(\mU^\top\mM) = \tr(\mU^\top\mA\mSigma \mB^\top)=\tr(\mP\mSigma),
 	\end{equation}
 	where $\mP=\mB^\top\mU^\top\mA$ is an orthogonal matrix since \[\mP\mP^\top=(\mB^\top\mU^\top\mA)(\mB^\top\mU^\top\mA)^\top=\mI.\]
 Thus every element including the diagonal of $\mP$ is no larger than 1. Then we have:
 	\begin{equation}
 	\label{eq:maxF1}
 	\tr(\mP\mSigma)\le \tr(\mSigma)
 	\end{equation}
 	On the other hand, when $\mU=\mA\mB^\top$, we have 
\[\tr(\mU^\top\mM)=\tr(\mB\mA^\top\mA\mSigma \mB^\top)=\mSigma.\]
 Thus  $\mU=\mA\mB^\top$ is the optimized solution to maximize the objective.
 \end{proof}
Accordingly, we have:
\begin{equation}
\label{eq:F}
\begin{aligned}
\mU({k+1})&=\argmin_{\mU^\top\mU=\mI} \langle \mU-\mU(k), \nabla_{\mU} h(\mU(k),\mV(k)) \rangle+\frac{\mu}{2}\|\mU-\mU(k)\|^2_F\\
&=\argmax_{\mU^\top\mU=\mI} \lg\mU,\mM(k)\rg = \mY\mZ^\top
\end{aligned}
\end{equation}
where $\mM(k):=2(\mX-\mU(k)\mV(k))\mV(k)^\top+\mu \mU(k)$ and $\mY,\mZ$ is obtained from $[\mY,\mSigma,\mZ] =\tmop{svd}(\mM(k))$.

Then we compute $\mV({k+1})$:
\begin{equation}
\label{eq:g}
\begin{aligned}
	\v_j({k+1})&= \argmin_{\|\v\|=1}\langle \v-\v_j(k), \nabla_{\v_j} h(\mU(k+1),\mV(k)) \rangle+\frac{\lambda}{2}\|\v-\v_j(k)\|^2\\
	&=\argmax_{\|\v\|=1} \langle \v,\q \rangle\\
	&=\frac{\q}{\|\q\|},\quad\tmop{for}\ j=1,\ldots,n,
\end{aligned}
\end{equation}
where $\q:=2\mU({k+1})^\top\x_j+(\lambda-2) \v_j(k)$.

\begin{algorithm}[tb]
	\caption{Alternating Linearized Minimization for Problem Eq.~(\getrefnumber{eq:h})}
	\label{alg:alg}
	\begin{algorithmic}
		\STATE {\bfseries Input:} data $\mX\in\R^{m\times n}$, rank of factors $r$, regularization parameters $\lambda,\mu$, number of iterations $K$
		\STATE {\bfseries Initialization:} $\mU(0)\in\R^{m\times r}, \mV(0)\in\R^{r\times n}$
		\STATE {\bf for} {$k=1,\ldots,K$}
		\STATE \qquad Optimize $\mU(k)$  via  Eq.~(\ref{eq:F})
		\STATE \qquad Optimize  $\mV(k)$  via  Eq.~(\ref{eq:g}) for $j=1,\ldots,n$

		\STATE {\bf end}
		\STATE {\bfseries Output:} $\mU(K)$ and $\mV(K)$
	\end{algorithmic}
\end{algorithm}

\section{Convergence Analysis}\label{sec:convergence}
In the following case, we let $\bbU$ and $\bbV$ be as defined in  Eq. \eqref{eq:setFG 2}, and show the convergence of our proposed algorithm in the last section.

To begin with, we first show that $h(\mU,\mV)$ has Lipschitz continuous gradient at $\mU\in\bbU,\mV\in\bbV$, which will be very useful for the following convergence analysis.
\begin{proposition}
$h(\mU,\mV)$ has Lipschitz continuous gradient at $\mU\in\bbU,\mV\in\bbV$, where $\bbU$ and $\bbV$ are defined in  Eq. \eqref{eq:setFG 2}. That is, there exists a constant $L_c$ such that
\begin{equation}
\|\nabla h(\mU,\mV) - \nabla h(\mU',\mV')\|_F \leq L_c\|(\mU,\mV) - (\mU',\mV')\|_F
\label{eq:Lipschitz h}
\end{equation}
for all $\mU,\mU'\in\bbU$ and $\mV,\mV'\in\bbV$. Here $L_c>0$ is referred to as the Lipschitz constant.
\label{prop:Lipschitz h}\end{proposition}
\begin{proof}[Proof of Proposition~\ref{prop:Lipschitz h}]

It is equivalent to show $\|\nabla^2 h(\mU,\mV)\| \leq L_c$ for all $\mU\in\bbU,\mV\in\bbV$. Standard computations give the Hessian quadrature form  $[\nabla^2 h(\mU,\mV)](\mDelta,\mDelta)$ for any $\mDelta = \begin{bmatrix} \mDelta_{\mU}\\ \mDelta_{\mV}^\top\end{bmatrix}\in\R^{(n+m)\times r}$ (where $\mDelta_{\mU}\in\R^{m\times r}$ and $\mDelta_{\mV}\in\R^{r\times n}$) as
\begin{equation}
\label{eq:hessian}\begin{split}
[\nabla^2h(\mU,\mV)](\mDelta,\mDelta)= \left\|\mDelta_{\mU}\mV+ \mU\mDelta_{\mV}\right\|_F^2  + 2\left\langle \mU\mV -\mX ,\mDelta_{\mU}\mDelta_{\mV} \right\rangle
\end{split}\end{equation}
which gives:
\begin{equation}
\label{eq:hessian_L}
\begin{aligned}
\|\nabla^2 h(\mU,\mV)\|  &= \maximize_{\|\mDelta\|_F = 1}\left| [\nabla^2h(\mU,\mV)](\mDelta,\mDelta) \right|\\
& \leq  \maximize_{\|\mDelta\|_F = 1}\left\|\mDelta_{\mU}\mV+ \mU\mDelta_{\mV}\right\|_F^2 + 2\left| \left\langle \mU\mV -\mX ,\mDelta_{\mU}\mDelta_{\mV} \right\rangle \right|\\
& \leq 2 (\|\mU\|^2_F + \|\mV\|^2_F + \|\mU\|_F\|\mV\|_F + \|\mX\|_F):= L_c,
\end{aligned}
\end{equation}
where the inequality follows from $|\langle \mA,\mB \rangle|\leq \|\mA\|_F\|\mB\|_F$ and  $\|\mC\mD\|_F\leq\|\mC\|_F\|\mD\|_F$. Due to the constraints on $\mU$ and $\mV$, we have $\|\mU\|^2_F = \tr(\mU^\top\mU) = \tr(\mI) = r, \|\mV\|^2_F = \sum_{j=1}^n\|\v_j\|^2=n$.
\end{proof}

To analyse the convergence, we rewrite Eq. \eqref{eq:h} as
\begin{equation}
\minimize_{\mU,\mV} f(\mU,\mV)= h(\mU,\mV) + \delta_{\bbU}(\mU) + \delta_{\bbV}(\mV),
\label{eq:obj no regularizer}
\end{equation}
where 
\[\delta_{\bbU}(\mU) = \begin{cases} 0, & \mU\in\bbU\\ \infty, & \mU\notin \bbU\end{cases}
\] is the indicator function of the set $\bbU$ and therefore nonsmooth, so is $\delta_{\bbV}(\mV)$. 

The following result establishes that the subsequence convergence property of the proposed algorithm, i.e., the sequence generated by Algorithm~\ref{alg:alg} is bounded and any of its limit point is a critical point of Eq. \eqref{eq:obj no regularizer}.
\begin{theorem}[Subsequence convergence]
Let $\{\mW(k)\}_{k\geq 0} =\{(\mU(k),\mV(k))\}_{k\geq 0}$ be the sequence generated by Algorithm~\ref{alg:alg} with constant step size $\lambda,\mu> L_c$. Then the sequence $\{\mW(k)\}_{k\geq 0}$ is bounded and obeys the following properties:
\begin{itemize}
\item[(P1):] Sufficient decrease:
\begin{align}
&f(\mW(k-1)) - f(\mW(k))\geq \frac{\min(\lambda,\mu) - L_c}{2}\| \mW(k)-\mW(k-1)\|_F^2,
\label{eq:sufficient decrease}\end{align}
which implies that
\begin{equation}
\lim_{k\rightarrow \infty}\|\mW(k-1)-\mW(k)\|_F= 0.
\label{eq:diff goes to 0}
\end{equation}
\item[(P2):] The sequence $\{f(\mW(k))\}_{k\geq 0}$ is convergent.
\item[(P3):] For any convergent subsequence $\{\mW(k')\}$, its limit point $\mW^\star$ is a critical point of $f$ and
    \begin{equation}
    \lim_{k'\rightarrow \infty}f(\mW(k')) = \lim_{k\rightarrow \infty}f(\mW(k)) = f(\mW^\star).
    \label{eq:lim f = f lim}
\end{equation}
\end{itemize}
\label{thm:subsequence convergence}
\end{theorem}
\begin{proof}[Proof of Theorem \ref{thm:subsequence convergence}] 
	Before proving Theorem \ref{thm:subsequence convergence}, we give out some necessary definition.
	\begin{definition}\cite{attouch2013convergence} Let $f:\R^d\rightarrow (-\infty,\infty]$ be a proper and lower semi-continuous function, whose domain is defined as
		\[
		\tmop{dom} f:=\left\{\u\in\R^n:f(\u)<\infty\right\}.
		\]
		
		The (Fr\'{e}chet) subdifferential $\partial f$ of $f$ at $\u$ is defined by
		\[
		\partial f(\u) = \left\{\z\in\R^d:\lim_{\v\rightarrow \u}\inf\frac{f(\v) - f(\u) - \langle \z, \v - \u\rangle}{\|\u - \v\|}\geq 0\right\}
		\]
		for any $\u\in \tmop{dom} h$ and $\partial f(\u) = \emptyset$ if $\u\notin \tmop{dom} f$.
		
		We say $\u$ is a limiting critical  point, or simply a critical point of $f$ if
		\[
		\zero \in \partial f(\u).
		\]
		\label{def:subdifferential}\end{definition}
	We now turn to prove Theorem \ref{thm:subsequence convergence}.
	
\begin{itemize}
	\item {\bf Showing (P1):}
	First note that for all $k$, according to our alternating minimization method, we always have $\delta_{\bbU}(\mU(k)) = \delta_{\bbV}(\mV(k)) = 0$ and thus $f(\mW(k)) = h(\mW(k))$.
	
	Since $h(\mU,\mV)$ has Lipschitz continuous gradient at $\mU\in \bbU,\mV\in\bbV$ with Lipschitz gradient $L_c$ and $\lambda>L_c$, we define $h_{L_c}(\mU,\mU',\mV)$ as proximal regularization of $h(\mU,\mV)$ linearized at $\mU',\mV$:
	\[
	h(\mU',\mV) + \langle \nabla_\mU h(\mU',\mV),\mU - \mU' \rangle + \frac{L_c}{2}\|\mU - \mU'\|_F^2,
	\]
	By the definition of Lipschitz continuous gradient and Taylor expansion, 
	we have
	\begin{equation}
	h(\mU,\mV)\leq h_{L_c}(\mU,\mU',\mV).
	\label{eq:Lip consequence}
\end{equation}
	Also by the definition of proximal map, we get:
	\begin{equation}
\begin{split}
		&\mU(k) 
		 = \argmin_{\mU} \delta_{\bbU}(\mU) +  \frac{\mu}{2}\|\mU - \mU(k-1)\|_F^2
		 +\langle \nabla_\mU h(\mU(k-1),\mV(k-1)),\mU - \mU(k-1) \rangle
	\end{split}\label{eq:opt F}
\end{equation}
	and hence we take $\mU(k)=\mU$, which implies that
	\begin{equation}
	\begin{aligned}
	\delta_{\bbU}(\mU(k)) +  \frac{\mu}{2}\|\mU(k) - \mU(k-1)\|_F^2 
	+\langle \nabla_\mU h(\mU(k-1),\mV(k-1)),\mU(k) - \mU(k-1) \rangle\leq \delta_{\bbU}(\mU(k-1))
	\label{eq: h lambda t}
\end{aligned}
	\end{equation}
	Combining Eq.~(\ref{eq:Lip consequence}) and Eq.~(\ref{eq: h lambda t}), we have:
	\begin{equation}\label{eq:FG_update}
	\begin{aligned}
		&h(\mU(k),\mV(k-1)) + \delta_{\bbU}(\mU(k))
\\
		& \leq h(\mU(k-1),\mV(k-1))+ \langle \nabla_\mU h(\mU(k-1),\mV(k-1)),\mU(k) - \mU(k-1) \rangle + \frac{L_c}{2}\|\mU(k) - \mU(k-1)\|_F^2\\
&\quad+\delta_{\bbU}(\mU(k))\\
		&\leq  h(\mU(k-1),\mV(k-1))+\frac{L_c}{2}\|\mU(k) - \mU(k-1)\|_F^2+\delta_{\bbU}(\mU(k-1))-\frac{\mu}{2}\|\mU(k) - \mU(k-1)\|_F^2\\
		&=h(\mU(k-1),\mV(k-1)) + \delta_{\bbU}(\mU(k-1)) - \frac{\mu - L_c}{2}\|\mU(k) - \mU(k-1)\|_F^2,
	\end{aligned}
	\end{equation}
	Similarly, we have
	\begin{equation}
\begin{aligned}
	\label{eq:G_update}
	h(\mU(k),\mV(k)) - h(\mU(k),\mV(k-1))+\delta_{\bbV}(\mV(k))-\delta_{\bbV}(\mV(k-1)) 
	&\leq -\frac{\lambda - L_c}{2}\|\mV(k) - \mV(k-1)\|_F^2
\end{aligned}
\end{equation}
	which together with the above equation gives Eq. (\ref{eq:sufficient decrease}). Now repeating Eq. (\ref{eq:sufficient decrease}) for all $k$ will give
	\begin{equation}
\label{eq:W_update}
	(\min(\lambda,\mu) - L_c)\sum_{k=1}^\infty\|\mW(k) - \mW(k-1)\|_F^2 \leq f(\mW(0)),
	\end{equation}
	which gives Eq. (\ref{eq:diff goes to 0}).
	\begin{remark}
		In our proposed algorithm, since in every update, our solution is closed while satisfying the constraints, thus in fact $\delta_{\bbU}$ and $\delta_{\bbV}$ are $0$, and $\infty$ is never achieved.
	\end{remark}
\item {\bf Showing (P2):} It follows from Eq. (16) that $\{f(\mW(k))\}_{k\geq 0}$ is a decreasing sequence. Due to the fact that $f$ is lower bounded as $f(\mW(k))\geq 0$ for all $k$, we conclude that $\{f(\mW(k))\}_{k\geq 0}$ is convergent.
	
\item {\bf Showing (P3):} Since $\mU(k')\in \bbU, \mV(k')\in \bbV$ for all $k'$ and both of the sets $\bbU$ and $\bbV$ are closed, we have $\mU^\star\in\bbU,\mV^\star\in\bbV$. Since $h$ is continuous, we have
	\begin{equation}	
\label{eq:stationary_point}\begin{split}
		\lim_{k'\rightarrow \infty}f(\mW(k'))& =  \lim_{k'\rightarrow \infty} h(\mU(k'),\mV(k')) + \delta_{\bbU}(\mU(k')) + \delta_{\bbV}(\mV(k'))
		= f(\mW^\star)
	\end{split}\nonumber
\end{equation}
	which together with the fact that  $\{f(\mW(k))\}_{k\geq 0}$ is convergent gives Eq. (\ref{eq:diff goes to 0}). 
	
	To show $\mW^\star$ is a critical point, we first consider Eq. \eqref{eq:opt F} and the optimality condition yields:
	\begin{equation}\label{eq:FF_point}
	\nabla_\mU h(\mU(k-1),\mV(k-1)) + \mu(\mU(k) - \mU(k-1)) + \partial \delta_\bbU(\mU(k))=0.
	\end{equation}
	Similarly, we have
	\begin{equation}\label{eq:GG_point}
	\nabla_\mV h(\mU(k),\mV(k-1)) + \lambda(\mV(k) - \mV(k-1)) + \partial \delta_\bbV(\mV(k))=0.
	\end{equation}
	Now, define
	\begin{align*}
	&\mA_k:={\nabla_\mU h(\mU(k),\mV_k) + \partial \delta_\bbU(\mU(k))},\\
	&\mB_k:={\nabla_\mV h(\mU(k),\mV(k)) + \partial \delta_\bbV(\mV(k))}.
	\end{align*}
	Thus, we have
	\begin{equation}
		\mA_k\in \partial_\mU f(\mU(k),\mV(k)), \mB_k \in \partial_\mV f(\mU(k),\mV(k)).
	\end{equation}
	It follows from the above that
	\begin{equation}\label{eq:A}
	\begin{aligned}
		\lim_{k\rightarrow \infty}\|\mA_k\|_F
		& \leq \lim_{k\rightarrow \infty} \|\nabla_\mU h(\mU(k),\mV(k)) - \nabla_\mU h(\mU(k-1),\mV(k-1))\|_F  + \mu \|\mU(k) - \mU(k-1)\|_F \\
		& \leq \lim_{k\rightarrow \infty} (L_c + \mu) \|\mW(k) - \mW(k-1)\|_F = 0.
	\end{aligned}
	\end{equation}
	Similarly, we have:
		\begin{equation}\label{eq:B}
	\begin{aligned}
		\lim_{k\rightarrow \infty}\|\mB_k\|_F
		 \leq \lim_{k\rightarrow \infty} (L_c + \lambda) \|\mW(k) - \mW(k-1)\|_F = 0.
	\end{aligned}
	\end{equation}
	Then we have:
	\begin{equation}\label{eq:all}
	\dist(\zero, \partial f(\mW(k)))\leq (2L_c + \mu+ \lambda) \|\mW(k) - \mW(k-1)\|_F
	\end{equation}
	Owing to the closedness properties of $\partial f(\mW(k'))$, we finally obtain 
\[\zero\in \partial f(\mW^\star).\] Thus, $\mW^\star$ is a critical point of $f$.
\end{itemize}
\end{proof}

\begin{theorem}[Sequence convergence] The sequence $\{\mW(k)\}_{k\geq 0}$  generated by  Algorithm~\ref{alg:alg} with a constant step size $\lambda, \mu> L_c$ is global-sequence convergence.
\label{thm:sequence convergence}
\end{theorem}
\begin{proof}[Proof of Theorem \ref{thm:sequence convergence}]
Before proving Theorem \ref{thm:sequence convergence}, we give out another important definition.
\begin{definition}[\bf Kurdyka-Lojasiewicz (KL) property]\cite{bolte2007lojasiewicz}\label{def:KL}
	We say a proper semi-continuous function $h(\vu)$ satisfies Kurdyka-Lojasiewicz (KL) property, if $\overline{\vu}$ is a   critical point of $h(\vu)$, then there exist $\delta>0,~\theta\in[0,1),~C_1>0$ such that
	\[
	\left|h(\vu) - h(\overline{\vu})\right|^{\theta} \leq C_1 \dist(\zero, \partial h(\vu)),~~\forall~\vu\in B(\overline{\vu}, \delta)
	\]
	
\end{definition}
We mention that the above KL property(also known as KL inequality) states the regularity of $h(\u)$ around its critical point $\u$ and the KL inequality trivially holds at non-critical point. There are a very large set of functions satisfying the KL inequality including any semi-algebraic functions \cite{attouch2013convergence}.  Clearly, the objective function $f$ is semi-algebraic as both $h$, $\delta_{\bbU}$ and $\delta_{\bbV}$ are semi-algebraic.

\begin{lemma}[Uniform KL property]\label{lem:KL:f}
	There exist $\delta_0>0,~{\theta_{KL}}\in[0,1),~C_{KL}>0$ such that for all $W$ s.t. $\operatorname{dist}(\mW,\mathbb{C}(\mW(0)))\leq\delta_0$:
	\begin{align}
	\label{eqn:KL:f:delta0}
	\left|f(\mW) - \overline f\right|^{{\theta_{KL}}} \leq C_{KL} \dist(\zero,\partial f(\mW))
	\end{align}	
	with $\overline f$ denoting the limiting function value  defined in (P2) of Theorem \ref{thm:subsequence convergence}.
\end{lemma}
\begin{proof}
	First we
	recognize the union $\bigcup_i B(\mW^\star_i, \delta_i)$ forms an open cover of $\mathbb{C}(\mW(0))$ with $\mW^\star_i$ representing all points in  $\mathbb{C}(\mW(0))$ and  $\delta_i$ to be chosen so that the the following KL property of $f$ at $\mW^\star_i\in \mathbb{C}(\mW(0))$ holds: 
	\begin{align*}
	\left|f(\mW) - \overline f\right|^{\theta_i} \leq C_i \dist(\zero,\partial f(\mW))~~\forall~\mW\in B(\mW^\star_i, \delta_i)
	\end{align*}	
	where we have used all $f(\mW^\star_{i})=\overline f$ by assertion (P3) of Theorem \ref{thm:subsequence convergence}.
	Then  due to the compactness of the set $\mathbb{C}(\mW(0))$, it has a finite subcover   $\bigcup_{i=1}^p B(\mW^\star_{k_i}, \delta_{k_i})$ for some positive integer $p$. Now combining all, we have for all $W\in \bigcup_{i=1}^p B(\mW^\star_{k_i},\delta_{k_i})$,
	\begin{align}\label{eqn:KL:f:individual}
	\left|f(\mW) - \overline f\right|^{{\theta_{KL}}} \leq C_{KL} \dist(\zero,\partial f(\mW))
	\end{align}
	with ${\theta_{KL}}=\max_{i=1}^p\{\theta_{k_i}\}$ and $C_{KL}=\max_{i=1}^p\{C_{k_i}\}$. Finally, since $\bigcup_{i=1}^p B(\mW^\star_{k_i},\delta_{k_i})$ is an open cover of $\mathbb{C}(\mW(0))$, there exists a sufficiently small number $\delta_0$ so that 
	\[
	\{(\mW):\operatorname{dist}(\mW,\mathbb{C}(\mW(0)))\leq\delta_0\}\subset \bigcup_{i=1}^p B(\mW^\star_{i},\delta_{k_i}).
	\]
	Therefore, \cref{eqn:KL:f:individual} holds whenever $\operatorname{dist}(\mW,\mathbb{C}(\mW(0)))\leq\delta_0$. 
\end{proof}
We now turn to prove Theorem \ref{thm:sequence convergence}.

According to Definition~\ref{def:KL}, there exists a sufficiently large $k_0$ satisfying:
\begin{align}
[f(\mW(k)) - f(\mW^\star)]^\theta \leq  C_2 \dist(\zero, \partial f(\mW(k))), \ \ \forall k\geq k_0.
\label{eqn:KL:f}
\end{align}
In the subsequent analysis, we restrict to $k\geq k_0$. 
Construct a concave function  $x^{1-\theta}$ for some $\theta\in[0,1)$ with domain $x>0$. Obviously, by the concavity, we have
\[  x_2^{1-\theta}-x_1^{1-\theta}\geq   (1-\theta) x_2^{-\theta}(x_2-x_1), \forall x_1>0,x_2>0\]
Replacing $x_1$ by $f(\mW_{k+1}) - f(\mW^\star)$ and  $x_2$ by $f(\mW_{k}) - f(\mW^\star)$ and using the sufficient decrease property, we have 
\begin{align*}
& [f(\mW(k)) - f(\mW^\star)]^{1-\theta} - [f(\mW(k+1)) - f(\mW^\star )]^{1-\theta} 
\\
& \geq  (1-\theta)\frac{f(\mW(k)) - f(\mW(k+1)) }{[f(\mW(k)) - f(\mW^\star)]^\theta}\\ 
&\geq  \frac{\lambda(1-\theta)}{2C_2} \frac{\|\mW(k) - \mW(k+1)\|_F^2 }{\dist(\zero,\partial f(\mW(k)))},\\
&\geq \frac{\lambda(1-\theta)}{2C_2 C_3} \frac{\|\mW(k) - \mW(k+1)\|_F^2 }{\|\mW(k) - \mW(k-1)\|_F}\\
&=\kappa (\frac{\|\mW(k) - \mW(k+1)\|_F^2 }{\|\mW(k) - \mW(k-1)\|_F}+\|\mW(k) - \mW(k-1)\|_F) -\kappa \|\mW(k) - \mW(k-1)\|_F\\
&\geq \kappa\left( 2 \|\mW(k) - \mW(k+1)\|_F - \|\mW(k) - \mW(k-1)\|_F\right) 
\end{align*}
And accordingly, we have:
\begin{equation}
\begin{aligned}
2 \|\mW(k) - \mW(k+1)\|_F - \|\mW(k) - \mW(k-1)\|_F
& \leq \beta \left([f(\mW(k)) - f(\mW^\star)]^{1-\theta} - [f(\mW(k+1)) - f(\mW^\star) ]^{1-\theta} \right)
\end{aligned}
\end{equation}
with $C_3:=2L_c + \mu+ \lambda, \kappa:= \frac{\lambda(1-\theta)}{2C_2 C_3}, \beta:=\left(\frac{\lambda(1-\theta)}{2C_2C_3}\right)^{-1}$.

Summing the above inequalities up from some $\widetilde k>k_0$ to infinity yields 
\begin{equation} \label{eq:difference summable}
\begin{aligned}
\sum_{k=\widetilde k}^{\infty}  \|\mW(k) - \mW(k+1)\|_F 
& \leq  \|\mW( \widetilde k) - \mW(\widetilde k-1)\|_F +  \beta  [f(\mW(\widetilde k)) - f(\mW^\star)]^{1-\theta}
\end{aligned}
\end{equation}
implying
\[
\sum_{k=\widetilde k}^{\infty}  \|\mW(k) - \mW(k+1)\|_F  < \infty.
\]
Following some standard arguments one can see that  
\[
\limsup_{t\rightarrow \infty, t_1,t_2\geq t}  \|\mW(t_1) - \mW(t_2)\|_F  = 0
\]
which implies that the sequence $\{  \mW(k) \}$  is Cauchy, and  hence  convergent.  Hence, the limit point set $\calC(\mW(0))$ is  singleton $\mW^\star$.
\end{proof}

\begin{theorem}[Convergence Rate] The convergence rate is at least sub-linear.
	\label{thm:convergence_rate}\end{theorem}
 Towards that end,   we first know from the above argument that $\{\mW(k)\}$ converges to some point $\mW^\star$, i.e., $\lim_{k\to\infty}\mW(k)=\mW^\star.$ Then using \Cref{eq:difference summable} and the triangle inequality, we obtain
 \begin{equation}\label{eq:convergence rate 1}
\begin{aligned}
\|\mW(\widetilde k) - \mW^\star\|_F &\leq\sum_{k=\widetilde k}^{\infty}  \|\mW(k) - \mW(k+1)\|_F  
\\
&\leq   \|\mW(\widetilde k) - \mW(\widetilde k-1)\|_F + \beta  [f(\mW(\widetilde k)) - f(\mW^\star)]^{1-\theta}
\end{aligned}
\end{equation}
which indicates the convergence rate of $\mW(\widetilde k) \rightarrow  \mW^\star$  is at least as fast as the rate that $ \|\mW(\widetilde k) - \mW(\widetilde k-1)\|_F + \beta  [f(\mW(\widetilde k)) - f(\mW^\star)]^{1-\theta}$ converges to 0.  
In particular, the second term $\beta  [f(\mW(\widetilde k)) - f(\mW^\star)]^{1-\theta}$ can be controlled:
\begin{equation}\begin{aligned}
\beta [f(\mW(\widetilde k)) - f(\mW^\star)]^{\theta}   
&\leq \beta C_2 \dist(\zero, \partial f(\mW(\widetilde k))) 
\\
&\leq \underbrace{\beta C_2  (2B_0+\lambda+\|\mX\|_F)}_{:=\alpha}\|\mW(\widetilde k)-\mW(\widetilde k-1)\|_F 
\label{eqn:bound:beta:f}
\end{aligned}
\end{equation}

Plugging \eqref{eqn:bound:beta:f} back to \eqref{eq:convergence rate 1}, we then have
\[
\sum_{k=\widetilde k}^{\infty}  \|\mW(k) - \mW(k+1)\|_F \leq \|\mW(\widetilde k) - \mW(\widetilde k-1)\|_F +\alpha \|\mW(\widetilde k) - \mW(\widetilde k-1)\|_F^{\frac{1-\theta}{\theta}}. 
\]

We divide the following analysis into two cases based on the value of the KL exponent $\theta$. 

\begin{itemize}
		\item \emph{Case I}: If $\theta = 0$, we set $\textit{Q}:=\{k\in\mathbb{N} : \mW(k+1) \neq \mW(k)\}$ and take $k$ in $\textit{Q}$. When $k$ is sufficiently large, then we have:
	\begin{equation}
	\|\mW(k+1)-\mW(k)\|_F^2 := C_4 > 0 
	\label{eqn:gap}
	\end{equation}
	On the other hand, 
	\begin{equation}
	\begin{aligned}
		f(\mW(k+1)) - f(\mW(k))&\geq \frac{\min(\lambda,\mu) - L_c}{2}\| \mW(k+1)-\mW(k)\|_F^2\\
		&=\frac{\min(\lambda,\mu) - L_c}{2} C_4
		\label{eq:theta0}
	\end{aligned}
	\end{equation}
	Since $f(\mW(k))$ is known to be converged to $0$, Eq.~(\ref{eq:theta0}) implies that $\textit{Q}$ is finite and sequence $\mW(k)$ converges in a finite number of steps.
	\item \emph{Case II}: $\theta\in (0,\frac{1}{2}]$. This case means $\frac{1-\theta}{\theta} \geq 1$. We define $P_{\widetilde k} = \sum_{i = \widetilde k}^\infty \|\mW_{i+1} - \mW_i\|_F$,
	\begin{equation}
	P_{\widetilde k} \leq P_{{\widetilde k}-1}- P_{\widetilde k} + \alpha \left[P_{\widetilde k-1}- P_{\widetilde k}\right]^{\frac{1-\theta}{\theta}}. \label{eq:convergence rate 2}
	\end{equation}
	Since $P_{\widetilde k-1}- P_{\widetilde k} \rightarrow 0$, there exists a positive integer $k_1$ such that $P_{{\widetilde k}-1}- P_{\widetilde k} < 1,~\forall~\widetilde k\geq k_1$.    Thus,
	\[
	P_{\widetilde k} \leq \left(1+ \alpha\right) (P_{\widetilde k-1}- P_{\widetilde k}),~~~~\forall ~ \widetilde k\geq \max\{ k_0,k_1\},
	\]
	which implies that
	\begin{equation}
	P_{\widetilde k} \leq \rho \cdot P_{\widetilde k-1},~~~~\forall ~ \widetilde k\geq \max\{ k_0,k_1\},
	\label{eq:Pk decay}
\end{equation}
	where $\rho = \frac{1+\alpha}{2+ \alpha} \in (0,1)$. This together with  \eqref{eq:convergence rate 1} gives the linear convergence rate
	\begin{equation}
\label{eq:linear convergence}
	\|\mW(k) - \mW^\star\|_F  \leq \calO( \rho^{k-\overline k} ), \ \forall \ k\geq \overline k. 
	\end{equation}
	where $\overline k = \max\{ k_0,k_1\}$.
	
	\item 
	\emph{Case III}: $\theta\in (1/2,1)$. This case means $\frac{1-\theta}{\theta} \leq 1$. Based on the former results, we have 
	\[
	P_{\widetilde k} \leq \left(1+ \alpha\right) \left[P_{\widetilde k-1}- P_{\widetilde k}\right]^{\frac{1-\theta}{\theta}},~~~~\forall ~ \widetilde k\geq \max\{ k_0,k_1\}.
	\]
	We now run into the same situation as in \cite{attouch2009convergence}. Hence  following  a similar argument  gives 
	\[
	P_{\widetilde k}^{\frac{1-2\theta}{1-\theta}} -  P_{\widetilde k-1}^{\frac{1-2\theta}{1-\theta}}  \geq \zeta, \ \forall \ k\geq \overline k  
	\]
	for some $\zeta >0$.  Then repeating and summing up the above inequality from $\overline k = \max\{ k_0,k_1\}$ to any $k> \overline k$, we can conclude 
	\[
	P_{\widetilde k}\leq \left[ P_{\widetilde k-1}^{\frac{1-2\theta}{1-\theta}} + \zeta(\widetilde k-\overline k)  \right]^{-\frac{1-\theta}{2\theta -1}} = \calO\left((\widetilde k-\overline k) ^{-\frac{1-\theta}{2\theta -1}} \right).
	\]
	Finally, the following sublinear convergence holds 
	\begin{equation} \label{eq:sublinear convergence}
	\|\mW(k) - \mW^\star\|_F  \leq \calO\left(( k-\overline k) ^{-\frac{1-\theta}{2\theta -1}} \right), \ \forall \  k> \overline k.
	\end{equation}
\end{itemize}

We end this proof by commenting that both linear and sublinear convergence rate are  closely related to  the KL exponent $\theta$ at the critical point $\mW^\star$.

\section{Experiments}
\label{exp}

In this section, we are going to apply our proposed spherical PCA to both synthetic data and real-world datasets to test the performance of our proposed method. The experiment on synthetic data will be introduced first followed by experiments on real-world datasets.
\subsection{Synthetic Data Experiment}

We first generate 200 data, half of which is distributed within the region between $X=Z$ and $Z$ axis (denoted as blue dots in the top part of Fig.~\ref{fig:toy}), while another group is generated within the region between $Y=Z$ and $Z$ axis (denoted as the red dots). These two clusters of data are generated through different angles. Thus when we do clustering, it should be angle distance rather than Euclidean distance to determine the clustering result. For our method, we learn a projection matrix $\mU\in \R^{3\times 2}$ and plot the component matrix $\mV\in \R^{2\times 200}$ as the bottom part illustrates. We see that, Euclidean distance-based method (such as $K$-means) will yield poor clustering result (middle part), while spherical-PCA will obtain good clustering result.
\begin{figure*}[h!]
	\centering
\begin{tabular}{ccc}	\includegraphics[height=4cm]{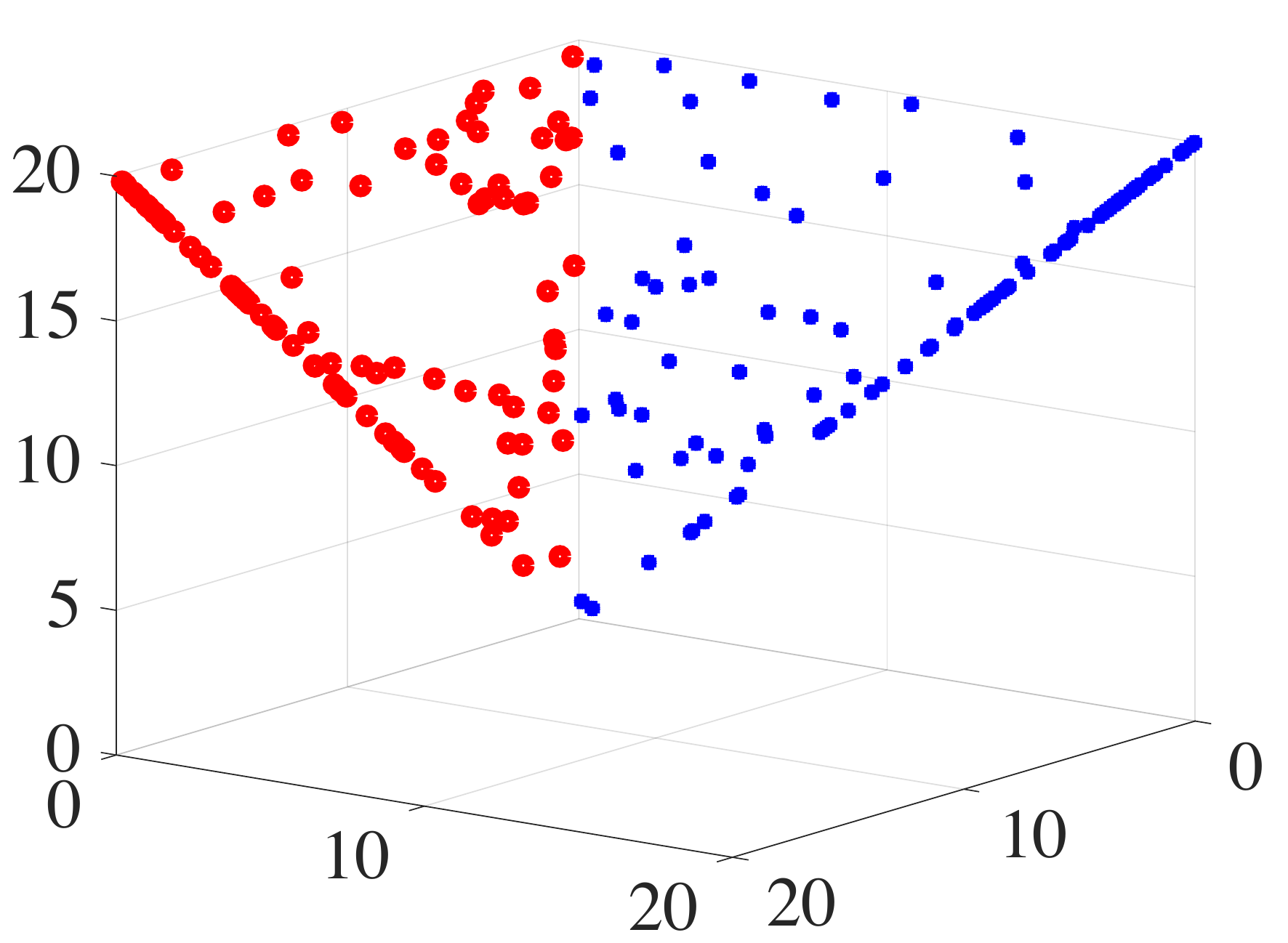}
&\includegraphics[height=4cm]{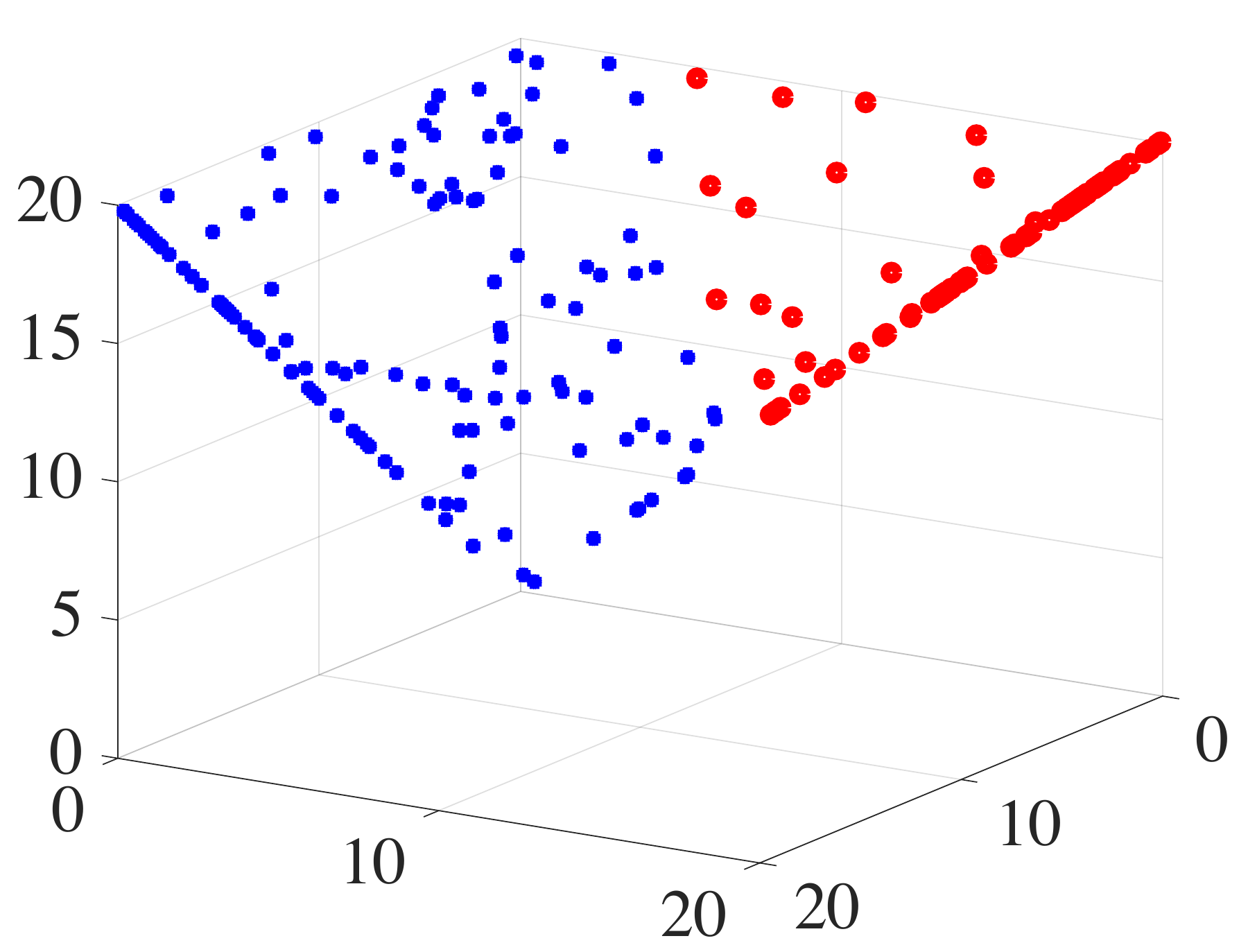}
&\includegraphics[height=4cm]{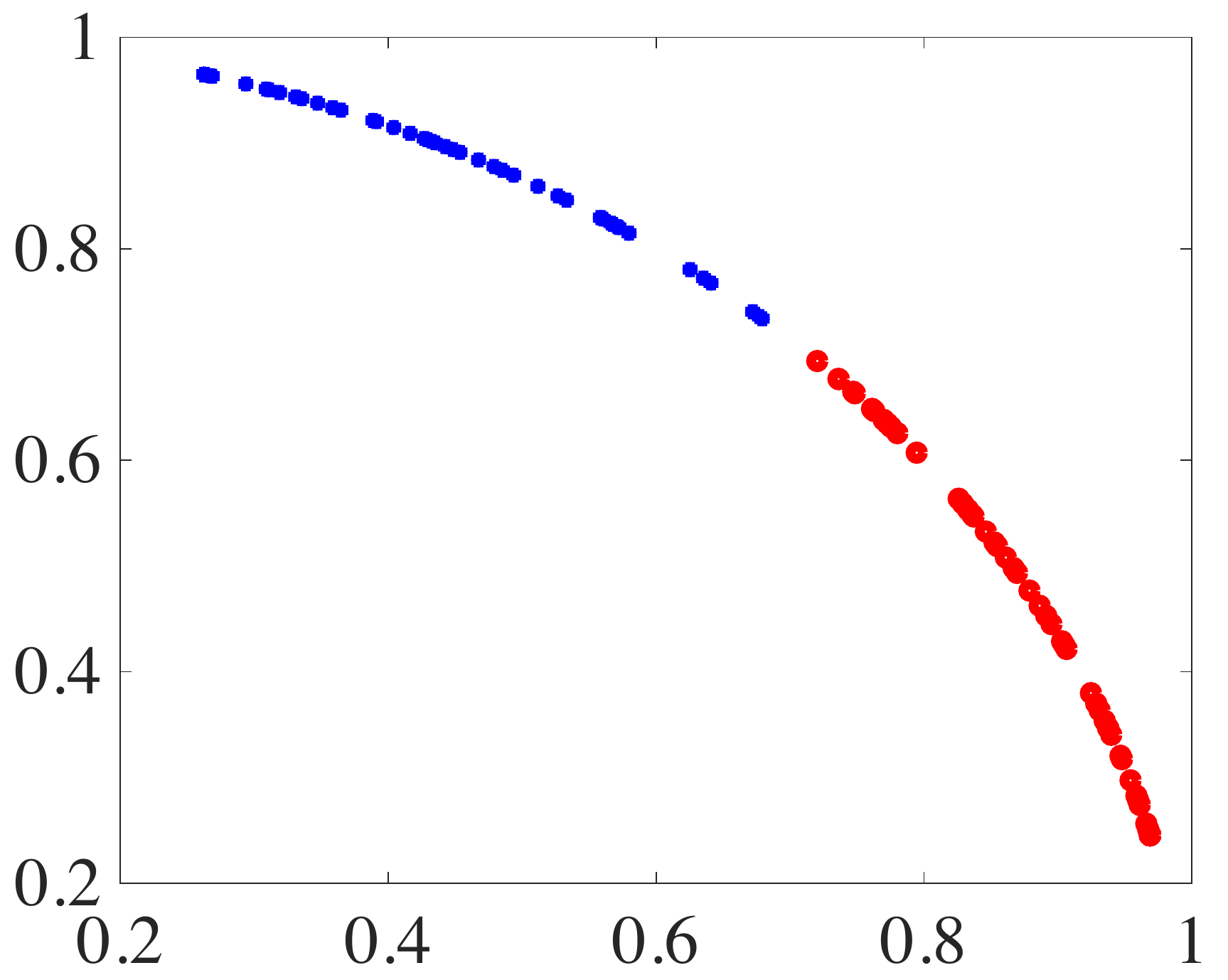}
\end{tabular}
	\caption{\textbf{Left}: two groups of data generated from two angles. \textbf{Middle}: clustering result with distance -based method $K$-means. \textbf{Right}: clustering result with our method. Blue and red represent different clusters.}
	\label{fig:toy}
\end{figure*}

Also, we show the convergence of $\{\mW(k)\}_{k\geq 0} =\{(\mU(k),\mV(k))\}_{k\geq 0}$ generated by our method. As Fig.~\ref{fig:change} shows, after short iterations, the generated sequences will be stable, which is in accordance with the convergence proof. It also illustrates the objective with update. We see that it converges fast with a sublinear rate, which validates our convergence rate analysis.  
\begin{figure*}[h!]
	\centering
	\begin{tabular}{ccc}			\includegraphics[height=4cm]{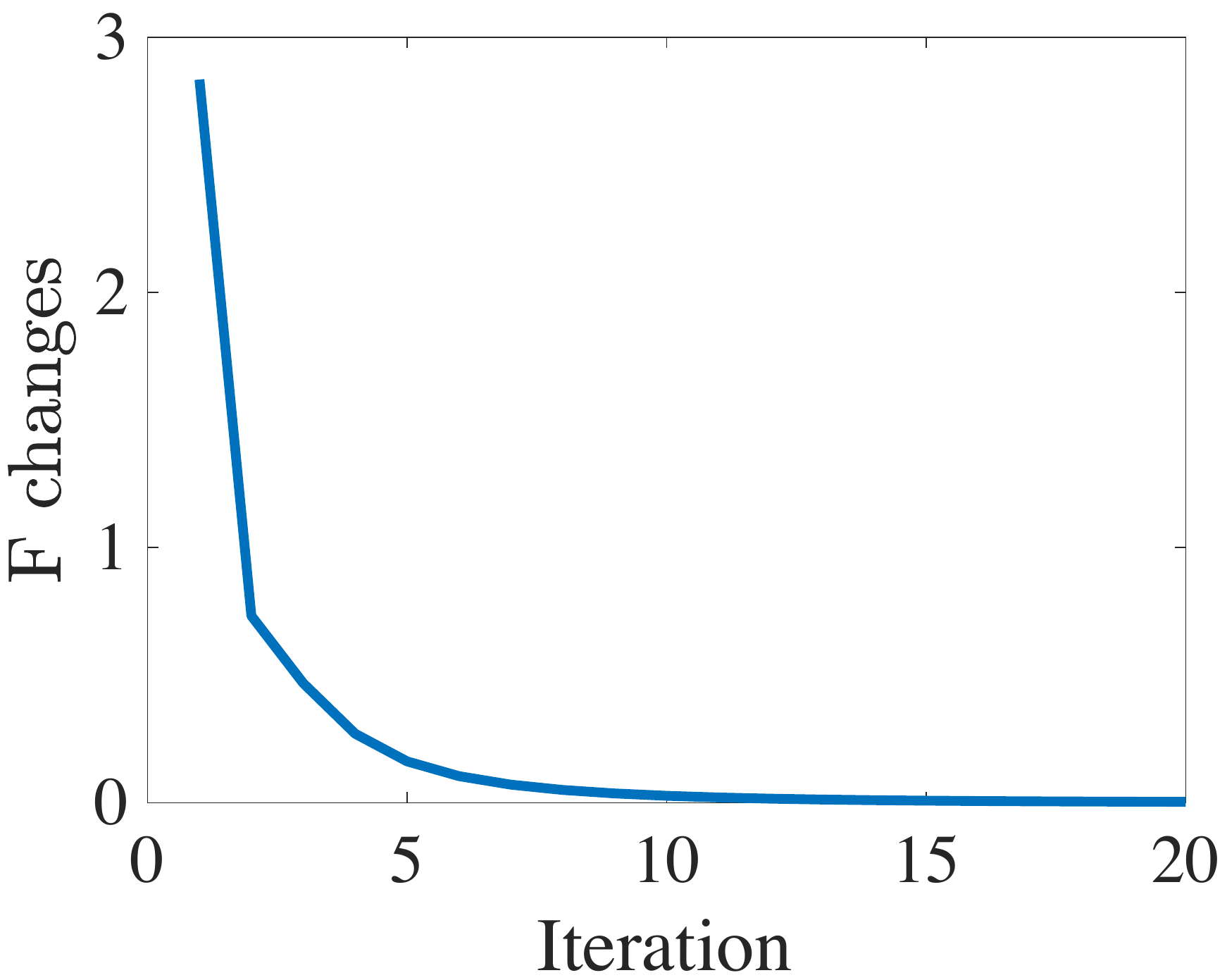}
&\includegraphics[height=4cm]{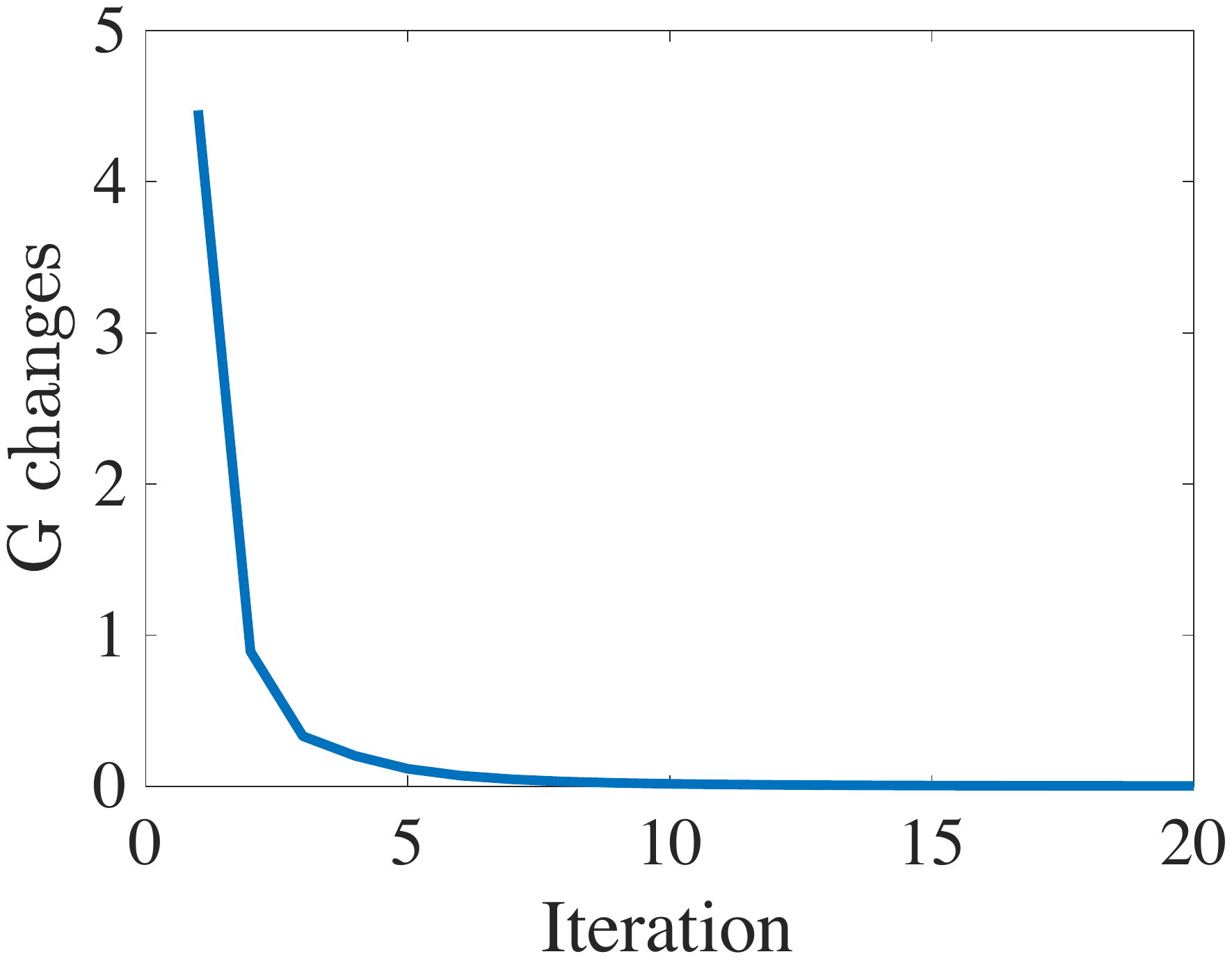}
&\includegraphics[height=4cm]{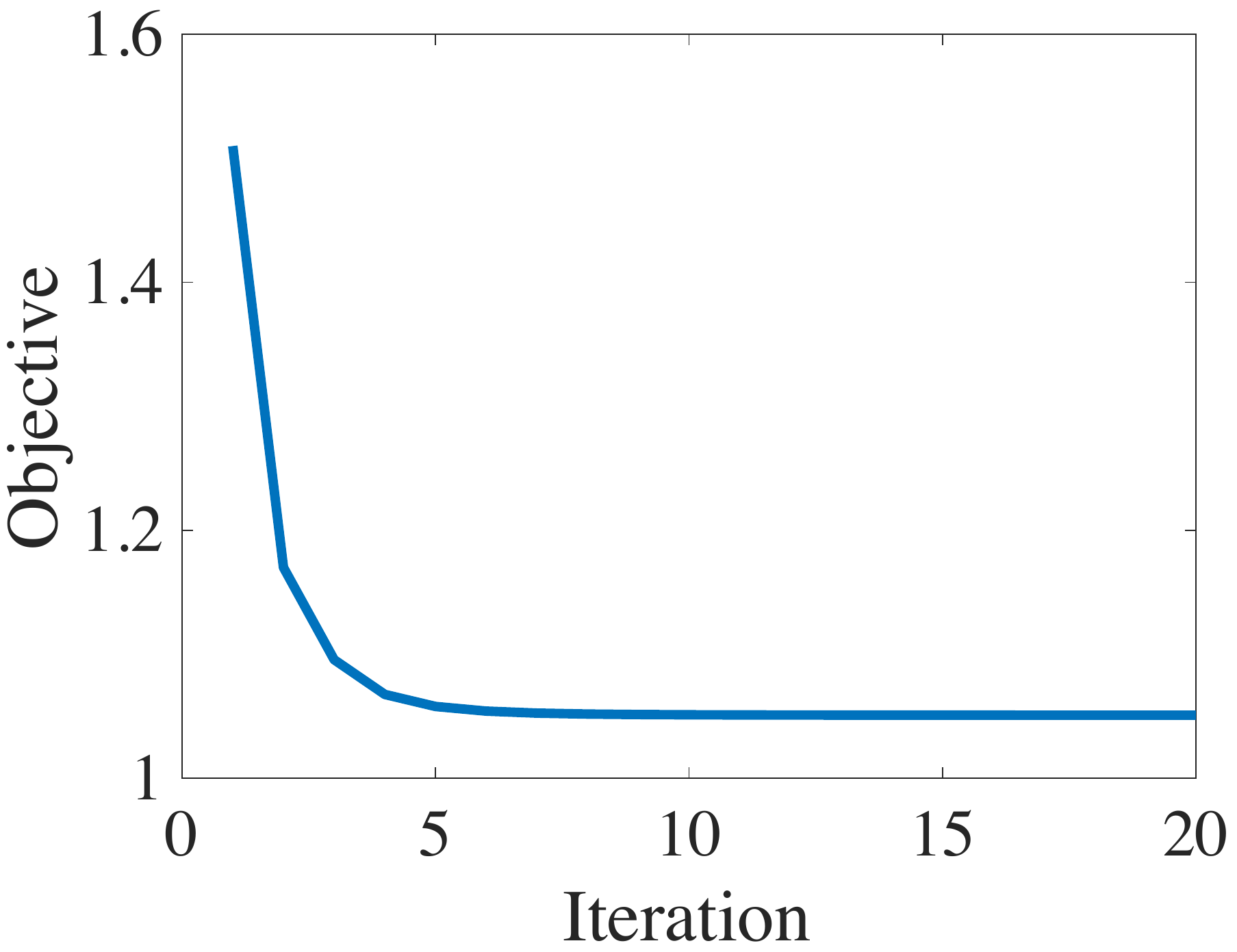}
\end{tabular}	
\caption{\textbf{Left}: $\|\mU(k+1)-\mU(k)\|_F$ with updates. \textbf{Center}: $\|\mV(k+1)-\mV(k)\|_F$ with updates. Both converge to $0$ after several iterations. \textbf{Right}: Objective converges at sub-linear rate. All validate our analysis.}
	\label{fig:change}
\end{figure*}

\subsection{Real-world Datasets Experiment}
\begin{table*}[h!]
	\begin{center}
		\caption{Clustering performance of different  algorithms on 20-newsgroup dataset}
		\label{tab:clustering_allDiff}
		\begin{tabular}{*{16}{c}}
			\toprule
			\multicolumn{1}{c}{Methods}  &
			\multicolumn{2}{c}{$K$-means}  &
			\multicolumn{2}{c}{MUA} &
			\multicolumn{2}{c}{PCA} &
			\multicolumn{2}{c}{R1-PCA} &
			\multicolumn{2}{c}{K-SVD} &
			\multicolumn{2}{c}{\textbf{Spherical PCA}}\\
			\cmidrule(r){1-1}\cmidrule(r){2-3}\cmidrule(r){4-5}\cmidrule(r){6-7}\cmidrule(r){8-9}
			\cmidrule(r){10-11}\cmidrule(r){12-13}
			{\#Groups}&{Acc} & {NMI} & {Acc} & {NMI} & {Acc} & {NMI} & {Acc} & {NMI} & {Acc} & {NMI} & 
			{Acc} & {NMI}  \\
			\midrule
			5&0.651&0.621&0.674&0.614& 0.703&0.628& 0.745&0.647 & 0.789&0.673 & \textbf{0.838}&\textbf{0.695} \\
			10&0.487&0.316& 0.478&0.320 & 0.502&0.383 & 0.535&0.398 & 0.527&0.394& \textbf{0.588}&\textbf{0.401}\\
			15&0.398&0.307& 0.387&0.301& 0.412&0.319& 0.423&0.320 & 0.461&0.377 & \textbf{0.486}&\textbf{0.385} \\
			20&0.315&0.242& 0.314&0.221& 0.362&0.248& 0.394&0.260 & 0.412&0.280 & \textbf{0.431}& \textbf{0.294}\\
			\bottomrule
		\end{tabular}
	\end{center}
\end{table*}
\begin{table*}[h!]
	\begin{center}
		\caption{Clustering performance of different  algorithms on four UCI datasets}
		\label{tab:clustering_allDif}
		\begin{tabular}{*{16}{c}}
			\toprule
			\multicolumn{1}{c}{Methods}  &
			\multicolumn{2}{c}{$K$-means}  &
			\multicolumn{2}{c}{MUA} &
			\multicolumn{2}{c}{PCA} &
			\multicolumn{2}{c}{R1-PCA} &
			\multicolumn{2}{c}{K-SVD} &
			\multicolumn{2}{c}{\textbf{Spherical PCA}}\\
			\cmidrule(r){1-1}\cmidrule(r){2-3}\cmidrule(r){4-5}\cmidrule(r){6-7}\cmidrule(r){8-9}
			\cmidrule(r){10-11}\cmidrule(r){12-13}
			{Data (\#class)}&{Acc} & {NMI} & {Acc} & {NMI} & {Acc} & {NMI} & {Acc} & {NMI} & {Acc} & {NMI} & 
			{Acc} & {NMI}  \\
			\midrule
			glass (6)&0.687&0.566&0.692&0.574& 0.732&0.608& 0.769&0.626 & \textbf{0.801}&\textbf{0.648} & 0.788&0.635 \\
			diabetes (2)&0.775&0.632& 0.788&0.654 & 0.761&0.613 & 0.808&0.631 & 0.827&0.672& \textbf{0.832}&\textbf{0.680}\\
			mfeat (10)&0.365&0.223& 0.358&0.211& 0.371&0.225& \textbf{0.431}&\textbf{0.342} & 0.412&0.328 & 0.425&0.330 \\
			isolet (26)&0.267&0.198& 0.253&0.181& 0.262&0.182& 0.324&0.201 & 0.357&0.246 & \textbf{0.373}&\textbf{0.250} \\
			\bottomrule
		\end{tabular}
	\end{center}
\end{table*}
It is known that in information retrieval, similarities or dissimilarities (proximities) between objects are more critical than Euclidean distance.
In this subsection, we will test our proposed method on
the widely-used 20-newsgroup dataset for clustering. We have different  newsgroups such as:
\textit{comp.graphics, rec.motorcycles, rec.sport.baseball,
sci.space, talk.politics.mideast}, etc.. 200 documents are randomly
sampled from each newsgroup.
The word-document matrix
X is constructed with 500 words selected according to the
mutual information between words and documents. \textit{Tf.idf}
term weighting is used before normalization. Clustering accuracy are computed
using the known class labels. Results will be compared including clustering accuracy (Acc.) and Normalized Mutual Information (NMI)~\cite{xu2003document}.

Different clustering algorithms will be compared including: 
\begin{enumerate}
	\item  \textbf{R1-PCA}, which proposes a rotational invariant $\ell_1$-norm PCA, where  a robust covariance matrix will
soften the effects of outliers~\cite{ding2006r};
\item \textbf{K-SVD}, which is an iterative method that alternates between sparse coding of the
examples based on the current dictionary and a process of updating
the dictionary atoms to better fit the data~\cite{aharon2006k};
\item \textbf{PCA}, i.e. the vanilla PCA method in Eq.~(\ref{eq:nmf_obj_pca2}) without the constraint on $\mV$, which will be Euclidean distance-based by default;
\item \textbf{NMF} Matrix Factorization proposed by~\cite{lee2001algorithms, liu2018high, liu2015robust, wang2011simultaneous} where $\mU$ and $\mV$ are obtained by Multiplicative Updating Algorithm with nonnegative constraint
\item \textbf{$K$-means} \cite{jain2010data}.
\end{enumerate}

We vary the number of clusters from $5$ to $10, 15$ and $20$. In each newsgroup, $200$ documents are randomly sampled, and we repeat for $10$ times by taking the average and report the clustering result as Table~\ref{tab:clustering_allDiff} demonstrates.

We see that our proposed method Spherical PCA can always achieve both higher clustering accuracy and normalized mutual information in text analysis.

We also compare our method with other methods on UCI datasets including: \textit{glass, diabetes, mfeat} and \textit{isolet}. Table~\ref{tab:clustering_allDif} illustrates the results. We see that though our method doesn't show the absolute advantage as on text, still the result is considerably good.

All the experiments indicate that our method can achieve good performance on both text and non-text datasets, showing its potential for broader application.

\section{Conclusion}
In this paper, we study spherical PCA where the direction matrix is orthonormal and the component vectors are assumed to lie in the unitary sphere. The benefit is obvious that it can make the angle distance equivalent to Euclidean distance.  Due to the nonconvexity of  objective function and constraints on the factors which are difficult to tackle, we propose an alternating linearized minimization method to derive the solution, which is proved to be sequence convergent. Moreover, we analyze the convergence rate which is validated by our experiments. The results on real-world datasets and synthetic data illustrate the superiority of our method.

\bibliography{Convergence}
\bibliographystyle{plain}
\end{document}